\documentclass[twoside,11pt]{article}

\usepackage{jmlr2_preprint}
\usepackage[utf8]{inputenc} 
\usepackage[T1]{fontenc}    
\usepackage{hyperref}       
\usepackage{url}            
\usepackage{booktabs}       
\usepackage{amsfonts}       
\usepackage{nicefrac}       
\usepackage{microtype}      
\usepackage{amsmath,amsthm}
\usepackage{graphicx}
\newcommand{\chara}{\mathcal{\mathfrak{a}}}
\newcommand{\charb}{\mathcal{\mathfrak{b}}}
\newcommand{\charc}{\mathcal{\mathfrak{c}}}
\newcommand{\charl}{\mathcal{\mathfrak{L}}}
\newcommand{\charr}{\mathcal{\mathfrak{R}}}
\newcommand{\chare}{\mathcal{\mathfrak{E}}}
\newcommand{\acceptable}{\smiley}
\newcommand{\rejected}{\frownie}
\usepackage{wasysym}
\usepackage[symbols,nogroupskip]{glossaries}
\usepackage{xcolor}
\usepackage{enumitem}
\newtheorem{theorem}{Theorem}[section]
\newtheorem{definition}[theorem]{Definition}
\newtheorem{lemma}[theorem]{Lemma}

\newtheorem{proposition}[theorem]{Proposition}
\newlist{abbrv}{itemize}{1}
\setlist[abbrv,1]{label=,labelwidth=1in,align=parleft,itemsep=0.1\baselineskip,leftmargin=!}

\jmlrheading{1}{xxxx}{xx}{xx}{xx}{stoginmali00a}{John Stogin and Ankur Mali and Lee Giles}

\ShortHeadings{Finite precision RNNs}{stogin, mali and giles}
\firstpageno{1}

\begin{document}

\title{A provably stable neural network Turing Machine}

\author{\name John Stogin*  \email jstogin@gmail.com \\
       \addr Independent Researcher\\
       Chicago\\
       IL, 60654, USA
       \AND
       \name Ankur Mali*  \email ankurarjunmali@usf.edu \\
       \addr Department of Computer Science and Engineering\\
       University of South Florida\\
       Tampa, FL 33620, USA
       \AND
       \name Lee Giles \email clg20@psu.edu \\
       \addr College of Information Sciences and Technology\\
       The Pennsylvania State University\\
       University Park, PA 16802, USA
       }


\maketitle
\begin{abstract}
We introduce a neural stack architecture, including a differentiable parametrized stack operator that approximates stack push and pop operations for suitable choices of parameters that explicitly represents a stack. We prove the stability of this stack architecture: after arbitrarily many stack operations, the state of the neural stack still closely resembles the state of the discrete stack. Using the neural stack with a recurrent neural network, we introduce a neural network Pushdown Automaton (nnPDA) and prove that nnPDA with finite/bounded neurons and time can simulate any PDA. Furthermore, we extend our construction and propose new architecture neural state Turing Machine (nnTM). We prove that differentiable nnTM with bounded neurons can simulate TM in real time. Just like the neural stack, these architectures are also stable. Finally, we extend our construction to show that differentiable nnTM is equivalent to Universal Turing Machine (UTM) and can simulate any TM with only \textbf{seven finite/bounded precision} neurons. This work provides a new theoretical bound for the computational capability of bounded precision RNNs augmented with memory.
\footnote{First two authors contributed equally}
\end{abstract}

\begin{keywords}
  Turing Completeness, Universal Turing Machine, Tensor RNNs, Neural Tape, Neural Stack, Stability, Finite Precision.
\end{keywords}

%




\section{Introduction}

In formal language theory, Chomsky (\cite{chomsky1959algebraic}) classified languages into four levels of increasing complexity. Associated to each level is a class of state machines, or automata, (\cite{sipser1996introduction}) suitably complex to recognize languages at that level. At the lowest level are the finite state automata (FSA) which can recognize regular languages. At the next level, are the Pushdown automata (PDA) which are FSA augmented with a stack. These can recognize context-free languages. At the highest level are Turing machines, which are FSA augmented with a tape. These can recognize all computable languages.

Although technically Turing complete, recurrent neural networks are most analogous to FSAs, the least powerful in the hierarchy. To properly represent more powerful state machines, an additional memory architecture, representing at least the capabilities of a stack, is required. By combining two stacks into a tape, one can then reach the most general Turing Machine.

A proper neural representation of a stack should have the following properties.
\begin{itemize}
\item The representation should be differentiable: stack operations should be learnable by training on data and tuning weights using a standard gradient descent approach.
\item The representation should be natural: the architecture should insert the researcher's prior that a stack is appropriate for solving the problem at hand.
\item The representation should be stable: the architecture should faithfully represent a discrete stack after arbitrarily many operations despite the imperfection of learned parameters.
\end{itemize}

Das (\cite{nndpa1998sun, das1992learning}) proposed a natural differentiable stack. This stack maintained a sequence of blocks with variable size (ranging between 0 and 1) and read a combination of characters based on the set of blocks comprising the top of the stack. So, for example, after pushing 0.8 of character $\chara$ and then pushing 0.9 of character $\charb$, the top of the stack would read $0.9\charb + 0.1\chara$, with the remaining portion of $\chara$ below the top of the stack. Then after popping with intensity 0.3, the top of the stack would read $0.6\charb + 0.4\chara$. This representation is both differentiable and natural. But is it stable?

Unfortunately it is not. Take, for example, a pair of push-then-pop operations, which should in theory leave the stack unchanged. If their intensities don't perfectly match, they leave the state of the stack permanently altered. For example, pushing 0.98 of $\chara$ and then popping 0.97 leaves 0.01 of $\chara$ left on the stack. After doing this 100 times, the top of the stack will then read $\chara$ entirely, despite the fact that $\chara$ was essentially popped just as many times as it was pushed. At this point, the stack no longer behaves at all like the discrete stack it was supposed to represent. This is an instability.

More recent work on memory-augmented neural networks (\cite{grefenstette2015learning, joulin2015inferring, mali2019neural, graves2014neural, graves2016hybrid}) provide computational approaches to differentiable stacks and tapes.  For instance \cite{joulin2015inferring} discretize and round their stack operator at test time to ensure it works stably on longer strings. \cite{grefenstette2015learning} suffer a similar issue as \cite{nndpa1998sun}, and furthermore their approach relies on an assumption, rather than a proof, of stability. Other forms of NNs such as the Neural Turing Machine \cite{graves2014neural} and the Differential neural computer \cite{graves2016hybrid} suffer from similar stability issues. These models are well designed for a gradient descent approach to learning, but in no way can guarantee stability.

Theoretically RNNs (\cite{siegelmann94}) and transformers (\cite{attn_turing}) are Turing complete. However these representations are unnatural: the tape size depends on the precision of the floating point values and the tape cannot be distinguished from the rest of the architecture. This is also true for less powerful theoretical constructions (\cite{gru2019pda}) such as Gated recurrent unit (GRU) networks that are shown to be equivalent to PDA.
Recently (\cite{chung2021turing}) showed that a recurrent neural network with bounded neurons is Turing Complete and represents two stacks by encoding it within the growing memory module with discrete variables, but this representation is not differentiable. These constructions (\cite{siegelmann94, attn_turing, chung2014empirical, gru2019pda}) do not guarantee stability with a differentiable operator.

The contribution of this paper is the introduction of differentiable, natural, and stable automata based on a neural stack. The neural stack and its associated parametrized stack operator lie at the core of this paper. Once the stack is shown to be stable, the nnPDA is introduced as a recurrent neural network together with the neural stack. Finally, a tape is represented as two neural stacks joined end-to-end and used to introduce the nnTM. Crucially, the stability implies that these automata can handle arbitrarily long strings, which is not valid for other theoretical results (\cite{chung2014empirical, attn_turing, siegelmann95} or computational models \cite{nndpa1998sun, graves2014neural, mali2020neural, joulin2015inferring}). This is an essential property for designing neuro-symbolic architectures, especially when extending to out-of-distribution samples. Finally, we observe in light of the Universal Turing Machine (UTM) work of \cite{neary2007four} that a small number (such as seven) of bounded neurons are required to simulate any TM with this architecture.

\section{Related Work}

Historically, grammatical inference (\cite{gold1967language}) has been at the core of language learnability and could be considered to be fundamental in understanding important properties of natural languages. 
A summary of the theoretical work in formal languages and grammatical inference can be found in (\cite{de2010grammatical}).
Applications of formal language work have led to methods for predicting and understanding sequences in diverse areas, such as financial time series, genetics and bioinformatics, and software data exchange (\cite{giles2001noisy,wieczorek2016use,exler2018grammatical}). 
Many neural network models take the form of a first order (in weights) recurrent neural network (RNN) including Long short term memory (LSTM), Gated Recurrent Unit (GRU) and other variants and have been taught to learn context free and context-sensitive counter languages (\cite{lstmcfg,das1992learning,boden2000context,tabor2000fractal,wiles1995learning,sennhauser2018evaluating,nam2019number,wang2019stateregularized,cleeremans1989finite,kolen1994recurrent,cleeremans1989finite,weiss2017extracting}).  However, from a theoretical perspective, RNNs augmented with an external memory have historically been shown to be more capable of recognizing context free languages (CFLs), such as with a discrete stack (\cite{das1993using,pollack1990recursive,sun1997neural}), or, more recently, with various differentiable memory structures (\cite{joulin2015inferring,grefenstette2015learning,graves2014neural,kurach2015neural,zeng1994discrete,hao2018context,yogatama2018memory,graves2016hybrid,le2019neural,mali2019neural,mali2020recognizing}). Despite positive results, prior work on CFLs was unable to achieve perfect generalization on data beyond the training dataset, highlighting a troubling difficulty in preserving long term memory (\cite{sennhauser2018evaluating,mali2020neural, mali2021recognizing, mali2021recognizing2, suzgun2019memoryaugmented}). 

This is also true for other sequential models such as Transformers \cite{bhattamishra2020ability}, since they are restricted due to other positional encoding and cannot go beyond a sequence length. Theoretically self-attention which is a core element for recent success behind transformers based variants cannot even recognize CFLs even using \textbf{infinite precision} (\cite{hahn2020theoretical}) in weights.

Prior work (\cite{omlin1996stable}) have shown that due to instability issues neural networks can struggle in recognizing simplest grammars such as regular grammars when tested on longer strings.  Given RNNs acts which should be the same for states in a stack RNN learning a PDA since a PDA is DFA controlling a stack. 
Early work primarily focused on constructing DFA in recurrent networks with hard-limiting neurons (\cite{alon1991efficient,horne1994bounds,minsky1967computation}, sigmoidal \cite{omlin1996stable,giles1993extraction,omlin1996extraction,omlin1996constructing} and radial-basis functions \cite{alquezar1995algebraic,frasconi1992injecting}). The importance and equivalence of differentiable memory over a discrete stack while learning is still unclear (\cite{joulin2015inferring,mali2020recognizing}). Recently, (\cite{merrill2020formal}) theoretically studied several types of RNNs and transformers analyzed there capability to understand various classes of languages with finite precision and time. In this paper we provide a new theoretical bound and show nnTM is Turing complete with bounded/finite precision and time.

\section{Vector Representation of a Pushdown Automaton}\label{pda_sec}

The state machine responsible for recognizing a context free grammar (CFG) is a pushdown automaton (PDA). In this section, we review the PDA and describe how it can be represented using vector and tensor quantities, which are analogous to weights in a neural network. A pushdown automaton is a deterministic finite state automaton (DFA) that controls a stack.\cite{hopcroft2pda} We review the DFA in \S\ref{fsa_sec}, the stack in \S\ref{pda_stack_sec}, and the PDA in \S\ref{pda_sub_sec}. An \textbf{example CFG} and a vector representation of a corresponding PDA is provided in \S\ref{example_CFG_sec}.

\subsection{The Deterministic Finite-State Automaton (DFA)}\label{fsa_sec}

Here, we review the DFA and provide a vector representation.
First we show classical version of DFA which is adopted from the prior work. We then introduce dynamic version and vector representation, which is much closer to state transition for the sequential models such as RNNs. 
Formally, a DFA is defined as a 5-tuple, which is represented as follows:

\begin{definition}\emph{(DFA, classical version)}
  A deterministic finite state automaton is a quintuple $(Q, \Sigma, \delta, q_0, F)$, where
  $Q = \{q_1,...,q_n\}$ is a finite set of states, 
  $\Sigma = \{\chara_1,...,\chara_m\}$ is a finite alphabet,
  $\delta:Q\times \Sigma\rightarrow Q$ is the state transition function, and 
  $q_0\in Q$ is a start state,
  $F\subset Q$ is a set of acceptable final states.
\end{definition}

A string of $l$ characters $\chara^1...\chara^l$ can be fed to the DFA, which in turn changes state according to the following rules.
$$q^{t=0} = q_0$$
$$q^{t+1} = \delta(q^t, \chara^{t+1}).$$
As we are mainly interested in dynamics, we will work in the context of an arbitrary string. 

The following simplified definition focuses on the dynamics of the DFA which will be useful in creating a connection between DFA state transition and RNN state transition.

\begin{definition}\emph{(DFA, dynamic version)}
  At any time $t$, the state of a finite-state automaton is
  $$q^t\in \{q_1,...,q_n\}.$$
  The next state $q^{t+1}$ is given by the dynamic relation
  $$q^{t+1} = \delta^{t+1}(q^t),$$
  where $\delta^{t+1}(q) = \delta(q, \chara^{t+1})$.
\end{definition}

To easily generalize to a neural network with tensor connections, we will further modify the above definition to obtain a vector representation of the DFA.

\begin{definition}\label{FSA_def_3}\emph{(DFA, vectorized dynamic version)}
  At any time $t$, the state of a deterministic finite-state automaton is represented by the state vector\footnote{The use of the bar in $\bar{Q}^t$ is in preparation for neural networks. As neural networks are the main focus of this paper, we will generally use $Q^t$ to represent the state vector for a neural network and $\bar{Q}^t$ to represent the idealized state vector introduced here.}
  $$\bar{Q}^t\in\{0,1\}^n,$$
  whose components $\bar{Q}^t{}_i$ are
  \begin{equation*}
    \bar{Q}^t{}_i = \left\{
      \begin{array}{ll}
        1 & q_i = q^t \\
        0 & q_i \neq q^t.
      \end{array}
    \right.
  \end{equation*}

  The next state vector $\bar{Q}^{t+1}$ is given by the dynamic relation
  $$\bar{Q}^{t+1} = W^{t+1}\cdot\bar{Q}^t,$$
  where the transition matrix $W^{t+1}$, determined by the transition tensor $W$ and the input vector $I^{t+1}$, is
  $$W^{t+1}=W\cdot I^{t+1}.$$
The input vector $I^{t+1}\in\{0,1\}^m$ has components
  \begin{equation*}
    I^{t+1}{}_j = \left\{
      \begin{array}{ll}
        1 & \chara_j = \chara^{t+1} \\
        0 & \chara_j \neq \chara^{t+1},
      \end{array}
    \right.
  \end{equation*}
  and the transition tensor $W$ has components
  \begin{equation*}
    W_i{}^{jk} = \left\{
      \begin{array}{ll}
        1 & q_i = \delta(q_j, \chara_k) \\
        0 & q_i \neq \delta(q_j, \chara_k). 
      \end{array}
    \right.
  \end{equation*}
  In component form, the dynamic relation can be rewritten as
  \begin{equation}\label{fsa_dynamics}
    \bar{Q}^{t+1}{}_i = \sum_{jk} W_i{}^{jk}I^{t+1}{}_k\bar{Q}^t{}_j.
  \end{equation}
\end{definition}

The reader may wish to check that all three definitions are consistent with each other. The first definition is popular in the literature, but the third is most analogous to a neural network.

\subsection{The stack}\label{pda_stack_sec}

We turn our attention to the stack. Based on chomsky hierarchy DFA can only model regular languages and to model complex languages such as Context-free languages we need access to memory. The computational model designed to learn and recognize CFG's is known as PDA, which is DFA augmented with stack.
A PDA is then defined as a 7-tuple, the extra two elements corresponding to the stack. 

In this section we will introduce stack and define various vectorized operator which will serve as important element to show comparison with differentiable stack. The differentiable operator will be introduced in section 4. The classic version of stack is shown below:
\begin{definition}\emph{(Stack, classical version)}
  A stack is a pair $(\Gamma, e)$, where $\Gamma=\{\charb_1,...,\charb_{m_2}\}$ is the stack alphabet and $e$ is the error character.
\end{definition}

At any given time, a stack holds a (possibly empty) string of characters. That is, its state can be represented by the following definition.

\begin{definition}
  Let $\Gamma^*$ represent the set of strings constructed by the following rules.
  \begin{itemize}
  \item $\epsilon\in\Gamma^*$,
  \item $\charc S\in \Gamma^*$ for all $\charc\in\Gamma$ and $S\in\Gamma^*$.
  \end{itemize}
  Here, $\epsilon$ represents the empty string.
\end{definition}

There are a few stack operators, which manipulate the character string held by the stack.

\begin{definition}\label{stack_ops_def}\emph{(Stack operators, classical version)}
  Define the set of stack operators
  $$op(\Gamma)=\{\pi_-,\pi_0,\pi_+(\charb_1),...,\pi_+(\charb_{m_2}))\},$$
  where $\pi_-$ represents a pop operation, $\pi_0$ is the identity operation, and $\pi_+$ is a family of push operators indexed by $\Gamma$. They act on strings in $\Gamma^*$ according to the following rules.
  $$\pi_-\epsilon = \epsilon$$
  $$\pi_-\charc S = S$$
  $$\pi_0S = S$$
  $$\pi_+(\charc)S = \charc S.$$
  Finally, define the read function $r$ according to the following rules.
  $$r(\epsilon) = e$$
  $$r(\charc S) = \charc.$$
\end{definition}

To easily generalize to a neural network, we will modify the above stack definition to obtain a vector representation.

\begin{definition}\label{stack_vector_encoding}\emph{(Stack, vectorized version)}

  Let
  $$\mathcal{B}_{(m_2)}=\{(1,0,...,0), (0,1,0,...,0), ..., (0,...,0,1)\}$$
  be the canonical basis of $\mathbb{R}^{m_2}$.
  
  A stack $\bar{K}$ is a vector-valued sequence
  $$\bar{K}=\{\bar{K}_0,\bar{K}_1,...\}, \hspace{.5in} \bar{K}_i\in\mathcal{B}_{(m_2)}\cup\{\vec{0}\},$$
  with the additional requirement that there exists an integer $s$, called the \emph{stack size}, such that
  $$\bar{K}_i\in \mathcal{B}_{(m_2)}\hspace{.5in} i< s$$
  and
  $$\bar{K}_i = \vec{0}\hspace{.5in} s \le i.$$
  The vector-valued sequence $\bar{K}$ can be determined from a stack state $k\in\Gamma^*$ by setting $\bar{K}_i$ to be the one-hot encoding of the $i$th character $r((\pi_-)^ik)$ of $k$, using $\vec{0}$ to represent $e$.
\end{definition}

For the vectorized stack, we group all the stack operators into a single parametrized operator as defined here.

\begin{definition}\label{stack_operator_def_1}\emph{(Parametrized stack operator)}
  For any scalar pair
  $$(\bar{p}_+,\bar{p}_-)\in \{(1,0),(0,1),(0,0)\}$$
  and vector
  $$\bar{C}\in\mathcal{B}_{(m_2)},$$
  define the stack operator $\bar{\pi}(\bar{p}_+,\bar{p}_-,\bar{C})$ by how it acts on the stack $\bar{K}$ according to the following relations.
  $$(\bar{\pi}(\bar{p}_+,\bar{p}_-,\bar{C})\bar{K})_{i\ge 1} = \bar{p}_+\bar{K}_{i-1}+\bar{p}_-\bar{K}_{i+1}+(1-\bar{p}_+-\bar{p}_-)\bar{K}_i$$
  $$(\bar{\pi}(\bar{p}_+,\bar{p}_-,\bar{C})\bar{K})_0 = \bar{p}_+\bar{C}+\bar{p}_-\bar{K}_1+(1-\bar{p}_+-\bar{p}_-)\bar{K}_0$$
\end{definition}

It is worth verifying that according to the above definition, the operator $\bar{\pi}(1,0,\bar{C})$ pushes the vector $\bar{C}$ onto the stack, the operator $\bar{\pi}(0,1,\bar{C})$ pops the stack, and the operator $\bar{\pi}(0,0,\bar{C})$ is the identity operator.

In tensor products that will be used momentarily, if any factor is zero, the entire expression becomes zero. Therefore, instead of using the top of the stack $\bar{K}$ directly, we will instead use the stack reading vector $\bar{R}$, which has one additional component that is $1$ when $\bar{K}_0=\vec{0}$ and $0$ otherwise.

\begin{definition}\label{Rbar_def}\emph{(Stack Reading)}
  Given the top vector of the stack $\bar{K}_0$ with dimension $m_2$, we define the stack reading vector $\bar{R}$ to be a vector of dimension $m_2+1$, with components
$$
\bar{R}_i=\left\{
\begin{array}{ll}
    (\bar{K}_0)_i & 0\le i<m_2 \\
    1-||\bar{K}_0||_{L^\infty} & i=m_2.
\end{array}\right.
$$
\end{definition}

\subsection{The Pushdown Automaton (PDA)}\label{pda_sub_sec}

Recall that a PDA is a DFA with a stack. The following definition of a PDA closely resembles Hopcroft's definition. \cite{hopcroft2pda}\footnote{There is a minor discrepancy in that Hopcroft's definition allows for an additional operation that corresponds to one pop followed by two push operations. This is mainly provided for convenience and is not an essential part of the definition.}

\begin{definition}\label{classic_pda_def}\emph{(PDA, classical version)}
  A pushdown automaton is a 7-tuple $(Q, \Sigma, \Gamma, \delta, q_0, e, F)$, where
  $Q=\{q_1,...,q_n\}$ is a finite set of states,
  $\Sigma=\{\chara_1,...,\chara_{m_1}\}$ is the finite input alphabet,
  $\Gamma=\{\charb_1,...,\charb_{m_2}\}$ is the finite stack alphabet,
  $\delta:Q\times\Sigma\times\Gamma\rightarrow Q\times op(\Gamma)$ is the transition function with $op(\Gamma)$ being the set of stack operators from Definition \ref{stack_ops_def},
  $q_0\in Q$ is the start state,
  $e$ is the stack error character,
  and $F\subset Q$ is a set of acceptable final states.
\end{definition}

Let $\delta_Q:Q\times\Sigma\times\Gamma\rightarrow Q$ and $\delta_{op}:Q\times\Sigma\times\Gamma\rightarrow op(\Gamma)$ represent the first and second part of the transition function $\delta$. Let $q\in Q$ and $k\in\Gamma^*$. A string of $l$ characters $\chara^1...\chara^l$ can be fed to the PDA, which changes state and stack state according to the following rules.
$$q^{t=0} = q_0$$
$$k^{t=0} = \epsilon$$
$$q^{t+1} = \delta_Q(q^t,\chara^{t+1},r(k^t))$$
$$k^{t+1} = \delta_{op}(q^t,\chara^{t+1},r(k^t))k^t.$$

As we are mainly interested in dynamics, we will work in the context of an arbitrary string. The following simplified definition focuses on the dynamics of the PDA.

\begin{definition}\emph{(PDA, dynamic version)}
  At any time $t$, the state of a pushdown automaton is
  $$(q^t,k^t)\in \{q_1,...,q_n\}\times \Gamma^*.$$
  The next state $(q^{t+1},k^{t+1})$ is given by the dynamic relations
  $$q^{t+1} = \delta_Q^{t+1}(q^t,k^t)$$
  $$k^{t+1} = \delta_{op}^{t+1}(q^t,k^t)k^t,$$
  where $\delta_Q^{t+1}(q,k) = \delta_Q(q,\chara^{t+1},r(k))$ and $\delta_{op}^{t+1}(q,k) = \delta_K(q,\chara^{t+1},r(k))$.
\end{definition}

To easily generalize to a neural network, we will further modify the above definition to obtain a vector representation of the PDA.

\begin{definition}\label{pda_vec_def}\emph{(PDA, vectorized dynamic version)}
  At any time $t$, the state of a pushdown automaton is represented by the pair
  $$(\bar{Q}^t, \bar{K}^t),$$
  where $\bar{Q}^t$ is a one-hot vector encoding the state $q^t$ as in Definition \ref{FSA_def_3} and $\bar{K}^t$ is a vector encoding of the stack $k^t$ according to Definition \ref{stack_vector_encoding}.

  The next state pair $(\bar{Q}^{t+1},\bar{K}^{t+1})$ is given by the dynamic relations
  \begin{align}
    \bar{Q}^{t+1} &= W_Q^{t+1}\cdot \bar{Q}^t \label{pda_dynamics_start_eqn} \\
    \bar{K}^{t+1} &= \bar{\pi}^{t+1}\bar{K}^t, \label{pda_dynamics_end_eqn}
  \end{align}
  where the transition matrix $W_Q^{t+1}$ and operator $\bar{\pi}^{t+1}$ are defined as follows.

  The transition matrix $W_Q^{t+1}$ is given  in component form by
  $$(W_Q^{t+1})_i{}^j = \sum_{k,l}(W_Q)_i{}^{jkl}\bar{R}^t{}_l I^{t+1}{}_k,$$
  where the components $(W_Q)_i{}^{jkl}$ are chosen to represent the transition function $\delta_Q$ as in Definition \ref{FSA_def_3} and the stack reading $\bar{R}^t$ is determined from $\bar{K}_0^t$ as described by Definition \ref{Rbar_def}.

  The operator $\bar{\pi}^{t+1}$ is given by
  $$\bar{\pi}^{t+1}= \bar{\pi}(\bar{p}_+^{t+1}, \bar{p}_-^{t+1}, \bar{C}^{t+1}),$$
  where the parameters $\bar{p}_\pm^{t+1}$ and vector $\bar{C}^{t+1}$ are
  $$\bar{p}_{\pm}^{t+1} = \sum_{j,k,l} (W_{p_\pm})^{jkl}\bar{R}^t{}_lI^{t+1}{}_k \bar{Q}^t{}_j$$
  $$\bar{C}^{t+1}{}_i = \sum_{j,k,l}(W_C)_i{}^{jkl}\bar{R}^t{}_lI^{t+1}{}_k \bar{Q}^t{}_j$$
  where the components $(W_{p_+})^{jkl}$, $(W_{p_-})^{jkl}$, and $(W_C)_i{}^{jkl}$ are chosen to represent the transition function $\delta_{op}$ the same way the components of $W_Q$ are chosen.

\end{definition}

The reader may wish to check that all three PDA definitions are consistent with each other, given appropriate choices for the tensors $W_Q$, $W_{p_+}$, $W_{p_-}$, and $W_C$. The following example may be helpful.

\subsection{An Example PDA}\label{example_CFG_sec}
For illustrative purposes, we give an example here of a vectorized PDA for the grammer of balanced parentheses. The goal is to classify strings with characters `(' and `)' based on whether they close properly. So for example, the strings $``()"$, $``(())"$, and $``(()())"$ are valid, while the strings $``)"$, $``(()"$, and $``(()))"$ are invalid. To avoid confusion, we'll use $\charl$ to represent the left parenthesis `(' and $\charr$ to represent the right parenthesis `)', and introduce $\chare$ to represent the end of the string. With these representations, the example valid strings are represented by $\charl\charr\chare$, $\charl\charl\charr\charr\chare$, and $\charl\charl\charr\charl\charr\charr\chare$, while the example invalid strings are represented by $\charr\chare$, $\charl\charl\charr\chare$, and $\charl\charl\charr\charr\charr\chare$. 

For tensors $W$ with indices $j$, $k$, and $l$, the $j$ index is paired with state input. Our PDA will have $n=2$ states representing either acceptable $\acceptable$ or rejected $\rejected$. We will represent these with two dimensional vectors the following way.
$$
\acceptable\rightarrow \begin{pmatrix} 1\\0\end{pmatrix}
\hspace{.5in}
\rejected\rightarrow \begin{pmatrix} 0\\1\end{pmatrix}
$$

For tensors $W$ with indices $j$, $k$, and $l$, the $k$ index is paired with the input character. Since the input character alphabet has size $m_1=3$, we will represent the input characters as vectors the following way.
$$
\charl\rightarrow \begin{pmatrix} 1\\0\\0\end{pmatrix}
\hspace{.5in}
\charr\rightarrow \begin{pmatrix} 0\\1\\0\end{pmatrix}
\hspace{.5in}
\chare\rightarrow \begin{pmatrix} 0\\0\\1\end{pmatrix}
$$

Finally, for tensors $W$ with indices $j$, $k$, and $l$, the $l$ index is paired with the stack reading vector. Our stack alphabet has size $m_2=1$ as a single character must be pushed each time $\charl$ is encountered and popped each time $\charr$ is encountered. Although the stack vector has dimension $m_2=1$, the reading from stack has dimension $m_2+1=2$, with the additional dimension representing an empty stack reading. The possible stack reading vectors are as follows.
$$
nonempty=\begin{pmatrix} 1\\0\end{pmatrix}
\hspace{.5in}
empty=\begin{pmatrix} 0\\1\end{pmatrix}
$$

\subsubsection{State Transition Rules}
Given a stack reading and input character, which fix $l$ and $k$ indices respectively, the $i$ and $j$ components of the transition tensor $(W_Q)_i{}^{jkl}$ can be interpreted as entries in an adjacency matrix that acts on the input state vector to give the output state vector.

If $\charl$ is encountered, we should not change state. Thus,
$$(W_Q)_i{}^{j00}=(W_Q)_i{}^{j01}=\begin{pmatrix} 1&0\\0&1\end{pmatrix}.$$
If $\charr$ is encountered, then we should not change state unless the stack is empty, in which case we should change from $\acceptable$ to $\rejected$. Thus,
$$
(W_Q)_i{}^{j10}=\begin{pmatrix} 1&0\\0&1\end{pmatrix}
\hspace{.5in}
(W_Q)_i{}^{j11}=\begin{pmatrix} 0&0\\1&1\end{pmatrix}.
$$
Finally, if $\chare$ is encountered, then we should not change state unless the stack is non-empty, in which case we should change from $\acceptable$ to $\rejected$. Thus,
$$
(W_Q)_i{}^{j20}=\begin{pmatrix} 0&0\\1&1\end{pmatrix}
\hspace{.5in}
(W_Q)_i{}^{j21}=\begin{pmatrix} 1&0\\0&1\end{pmatrix}.
$$

\subsubsection{Stack Action Rules}

The only case where we push is when encountering $\charl$. Thus,
$$
(W_{p_+})^{jkl}=\left\{
\begin{array}{ll}
    1 & k=0 \\
    0 & otherwise
\end{array}\right.
$$
The only case where we pop is when encountering $\charr$. Thus,
$$
(W_{p_-})^{jkl}=\left\{
\begin{array}{ll}
    1 & k=1 \\
    0 & otherwise
\end{array}\right.
$$
It does not matter what happens to the stack when encountering $\chare$, so for simplicity we may set
$$(W_{p_+})^{j2l}=(W_{p_-})^{j2l}=0.$$

Finally, there is only one character in the stack alphabet, so for simplicity we may set
$$(W_C)_i{}^{jkl} = 1.$$
Alternatively, if we only which to specify a character when pushing, we may instead set
$$
(W_C)_i{}^{jkl}=\left\{
\begin{array}{ll}
    1 & k=0 \\
    0 & otherwise.
\end{array}\right.
$$
Either way, the index $i$ must take the value $0$ as the stack alphabet has size $m_2=1$.

\section{Neural Network with Stack Memory}\label{model_sec}

This section is the heart of the paper. In \S\ref{continuous_stack_sec}, we define the differentiable stack, and then in \S\ref{nnpda_def_sec}, we use it to define the neural network pushdown automaton (nnPDA). Finally, in \S\ref{nnpda_stability_sec}, we prove that the nnPDA stably approximates the PDA.

We will use the logistic sigmoid function
$$h_H(x) = \frac{1}{1+e^{-Hx}}$$
as an activation function. Note that $h_H(0)=\frac12$ and $h_H(x)$ decreases to $0$ as $Hx\rightarrow -\infty$ and increases to $1$ as $Hx\rightarrow\infty$. The scalar $H$ is a sensitivity parameter. In practice, we will often show that $x$ is bounded away from $0$, and then assume $H$ to be a positive constant sufficiently large so that $h_H(x)$ is as close as desired to either $0$ or $1$, depending on the sign of the input $x$. (See Lemma \ref{stack_op_lem_2}.) The reader may notice that a few proofs vaguely require that $H$ be sufficiently large. By taking the maximum of all lower bounds for $H$ required by each proof, we may arrive at a value for $H$ that satisfies all of these requirements simultaneously.

Furthermore whenever we refer differentiable operator is closer to ideal/discrete operator or approximates the operator, we refer to fixed point mapping or asymptotic stability. Fixed point of mapping function is defined as follows:
\begin{definition}
Let us assume f : Z $\rightarrow Z$ be a mapping on metric space. Thus a point $z_f$ $\in$ Z is called a fixed point of the mapping such that f($z_f$) = $z_f$
\end{definition}

To achieve this we define stability of $z_f$ as follows:
\begin{definition}
A fixed point $z_f$ is considered stable if there exists an range $R =[i,j] \in Z$ such that $z_f$ $\in$ R and iterations for f starts converging towards $z_f$ for any starting point $z_s$ $\in$ R.

\end{definition}

The continuous function Z $\rightarrow Z$ as the following useful property.

\begin{theorem}
(BROUWER’S FIXED POINT THEOREM) (\cite{boothby1971two}) Under a continuous
mapping function  f : Z $\rightarrow Z$ there exists at least one fixed point.
\end{theorem} 
Later (\cite{omlin1996stable}) showed that for sigmoid function their exists at-least one fixed point such that continuous or differentiable operator either converges to 0 or 1 and solution within that bound always stays stable.
In next section we will define our differentiable stack and their associated operators. 

\subsection{The differentiable stack}\label{continuous_stack_sec}
Here we define a differentiable stack memory~\cite{sun1997neural,grefenstette2015learning,joulin2015inferring}.
\begin{definition}\label{continuous_stack_def}\emph{(Differentiable stack)}
  For a stack alphabet of size $m_2$, a differentiable stack $K$ is a vector-valued sequence
  $$K=\{K_0,K_1,...\},\hspace{.5in} K_i\in [0,1]^{m_2},$$
  with the additional requirement that there exists an integer $s$, called the \emph{stack size}, such that
  $$K_i=\vec{0}\hspace{.25in} \text{ for all } \hspace{.25in}i \ge s.$$
\end{definition}
Note that the vectors $K_i$ in this definition closely resemble the idealized one-hot vectors $\bar{K}_i$ introduced in Definition \ref{stack_vector_encoding}.

While a traditional stack can be modified by distinct push and pop operations, here we introduce a single differentiable stack operator that, for particular parameter values, can closely resemble either push or pop.
\begin{definition}\label{continuous_stack_operator_def}\emph{(Differentiable stack operator)}
  For scalars $p_+$ and $p_-$ and vector $C\in[0,1]^n$, we define the stack operator $\pi(p_+,p_-,C)$ as follows.
  \begin{equation*}
    ((\pi(p_+,p_-,C)K)_{i\ge 1})_j
    = h_H\left(p_+(K_{i-1})_j+p_-(K_{i+1})_j+(1-p_+-p_-)(K_i)_j-\frac12\right),
  \end{equation*}
  \begin{equation*}
    ((\pi(p_+,p_-,C)K)_{i=0})_j 
    = h_H\left(p_+C_j+p_-(K_1)_j+(1-p_+-p_-)(K_0)_j-\frac12\right).
  \end{equation*}
\end{definition}
This is visualized in Figure \ref{continuous_stack_fig}. Note that the operator $\pi(0,1,C)$ is an differentiable pop operator, the operator $\pi(0,0,C)$ is an differentiable identity operator, and the operator $\pi(1,0,C)$ is an differentiable push operator that pushes the vector $C$ onto the stack, such that differentiable stack operator are nearly equal to ideal stack operators.

\begin{figure}
  \caption{\textbf{Differentiable Stack Operation} Visualization of the operator $\pi(p_+,p_-,C)$ with $p_+=0.7$ and $p_-=0.2$. First, a linear combination of the adjacent stack vectors is computed, then $h_H$ is applied to each component. The stack alphabet size is $m_2=4$. The stack size is $s=5$ before the operator is applied and $s=6$ afterward. As long as $p_+>0$, the stack size will increase by $1$.}
  \includegraphics[width=0.8\textwidth]{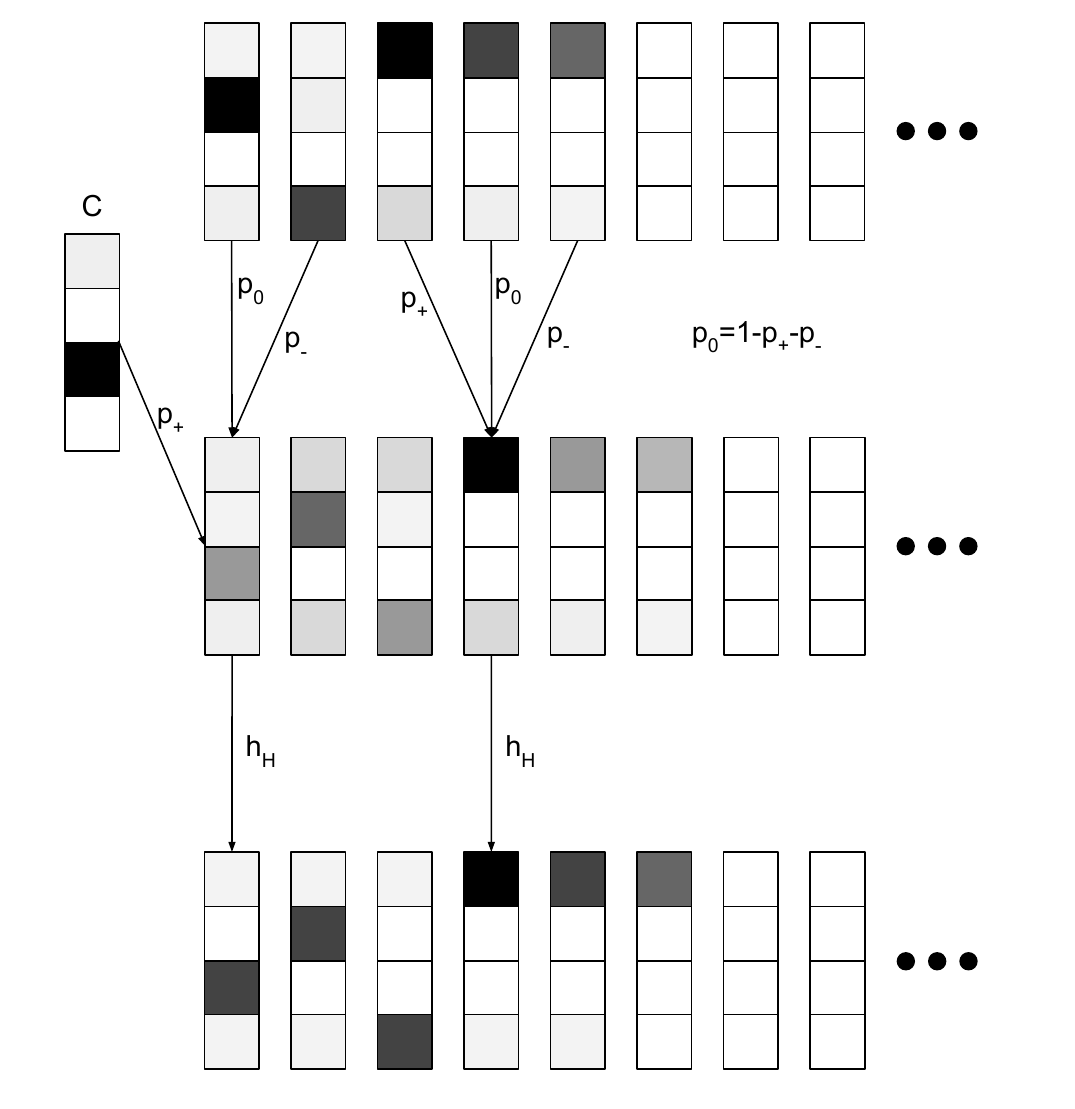}\centering
  \label{continuous_stack_fig}
\end{figure}

Recall Definition \ref{stack_operator_def_1}, wherein a similar idealized operator $\bar{\pi}(\bar{p}_+,\bar{p}_-,\bar{C})$ is defined, for which $\bar{\pi}(0,1,\bar{C})$ is exactly a pop operator, the operator $\bar{\pi}(0,0,C)$ is the identity operator, and the operator $\bar{\pi}(1,0,\bar{C})$ is exactly a push operator that pushes the vector $\bar{C}$ onto the stack. The relation between these two operators is formalized by the following proposition, which serves as a theoretical justification for $\pi(p_+,p_-,C)$.

\begin{proposition}\label{pibar_pi_prop}\emph{($\pi(\cdot,\cdot,\cdot)$ approximates $\bar{\pi}(\cdot,\cdot,\cdot)$)}
  
  Let $K$ be a stack according to Definition \ref{continuous_stack_def} and let $\bar{K}$ be an idealized stack according to Definition \ref{stack_vector_encoding}, and suppose that
  $$\sup_i||K_i-\bar{K}_i||_{L^\infty} \le \epsilon$$
  for some $\epsilon>0$.
  
  Let $\pi(p_+,p_-,C)$ be an operator according to Definition \ref{continuous_stack_operator_def} and $\bar{\pi}(\bar{p}_+,\bar{p}_-,\bar{C})$ be an idealized operator according to Definition \ref{stack_operator_def_1}. Suppose furthermore that, in addition to belonging to the domains of their respective operators, the quantities $p_+$, $p_-$, $C$, $\bar{p}_+$, $\bar{p}_-$, and $\bar{C}$ satisfy
  $$|p_+-\bar{p}_+| \le \epsilon$$
  $$|p_--\bar{p}_-| \le \epsilon$$
  $$||C-\bar{C}||_{L^\infty} \le \epsilon.$$
  If $\epsilon$ is sufficiently small, then $H$ can be chosen sufficiently large, depending only on $\epsilon$, so that
  $$\sup_i||(\pi(p_+,p_-,C)K)_i-(\bar{\pi}(\bar{p}_+,\bar{p}_-,\bar{C})\bar{K})_i||_{L^\infty} \le \epsilon.$$
\end{proposition}

To simplify the proof of this proposition, we will first state and prove two lemmas. The first lemma says that if a vector is within a fixed point of its corresponding idealized vector \footnote{Based on fixed point analysis, If differentiable operator or continuous function is within the fixed }, then applying the activation function will make the resulting vector very close to the idealized vector. (This requires a suitably large choice of the sensitivity parameter.) 
\begin{lemma}\label{stack_op_lem_2}
  Let $\bar{V}$ be a vector with components $\bar{V}_i\in\{0,1\}$ and let $V$ be a vector satisfying
  $$||V-\bar{V}||_{L^\infty} \le \epsilon_0$$
  for some $\epsilon_0< \frac12$.
  
  Then for all sufficiently small $\epsilon > 0$, for all $H$ sufficiently large depending on $\epsilon_0$ and $\epsilon$,
  $$\max_i\left|\bar{V}_i -h_H\left(V_i-\frac12\right)\right| \le \epsilon.$$
\end{lemma}
\begin{proof}
  Since $\epsilon_0<\frac12$, choose $H$ sufficiently large so that
  $$h_H\left(-\left(\frac12-\epsilon_0\right)\right) = 1-h_H\left(\frac12-\epsilon_0\right) \le \epsilon.$$
  
  Note that
  $$\left|\bar{V}_i -h_H\left(V_i-\frac12\right)\right| = \left|\bar{V}_i -h_H\left(\bar{V}_i-\frac12+(V_i-\bar{V}_i)\right)\right|.$$
  Pick an arbitrary index $i$. If $\bar{V}_i=1$,
  $$
    \left|\bar{V}_i-h_H\left(V_i-\frac12\right)\right| = 1-h_H\left(\frac12+(V_i-\bar{V}_i)\right)
                                                       \le 1-h_H\left(\frac12-\epsilon_0\right)
    \le \epsilon.
  $$
  If instead $\bar{V}_i=0$,
  $$
    \left|\bar{V}_i-h_H\left(V_i-\frac12\right)\right| = h_H\left(-\frac12+(V_i-\bar{V}_i)\right)
                                                       \le h_H\left(-\left(\frac12-\epsilon_0\right)\right)
                                                       \le \epsilon.
  $$
\qed
\end{proof}

Next we introduce a lemma that helps us relate expressions found in the definitions of $\pi$ and $\bar{\pi}$.
\begin{lemma}\label{stack_op_lem_1}
  Let $X,\bar{X},Y,\bar{Y},Z,\bar{Z}$ be vectors and let $x,\bar{x},y,\bar{y}$ be scalars. Let $e_d$ denote the difference between differentiable and discrete/ideal operator. Suppose that all scalars and all vector components lie in the range $[0,1]$. Suppose furthermore that
  $$||X-\bar{X}||_{L^\infty}\le \epsilon$$
  $$||Y-\bar{Y}||_{L^\infty}\le \epsilon$$
  $$||Z-\bar{Z}||_{L^\infty}\le \epsilon$$
  $$|x-\bar{x}|\le \epsilon$$
  $$|y-\bar{y}|\le \epsilon.$$
  Then
  $$xX+yY+(1-x-y)Z = \bar{x}\bar{X}+\bar{y}\bar{Y}+(1-\bar{x}-\bar{y})\bar{Z} + e_d,$$
  and
  $$||e_d||_{L^\infty}\le 7\epsilon.$$
\end{lemma}
\begin{proof}
  Let
  \begin{align*}
    e_d &= ({e_d})_X+({e_d})_Y+({e_d})_Z \\
    ({e_d})_X &= xX-\bar{x}\bar{X} \\
    ({e_d})_Y &= yY-\bar{y}\bar{Y} \\
    ({e_d})_Z &= (1-x-y)Z-(1-\bar{x}-\bar{y})\bar{Z}.
  \end{align*}
  It suffices to estimate each of these error terms individually.
  \begin{align*}
    ||({e_d})_X||_{L^\infty} &= ||xX-\bar{x}\bar{X}||_{L^\infty} \\
                         &= ||(x-\bar{x})\bar{X}+x(X-\bar{X})||_{L^\infty} \\
                         &\le |x-\bar{x}|||\bar{X}||_{L^\infty}+|x|||X-\bar{X}||_{L^\infty} \\
                         &\le \epsilon||\bar{X}||_{L^\infty}+|x|\epsilon \\
                         &\le 2\epsilon.
  \end{align*}
  By an analogous calculation, we also conclude
  $$||({e_d})_Y||_{L^\infty} \le 2\epsilon.$$
  And
  \begin{align*}
    ||({e_d})_Z||_{L^\infty} &= ||(1-x-y)Z-(1-\bar{x}-\bar{y})\bar{Z}||_{L^\infty} \\
                         &= ||(\bar{x}-x)\bar{Z}+(\bar{y}-y)\bar{Z}+(1-x-y)(Z-\bar{Z})||_{L^\infty} \\
                         &\le |\bar{x}-x|||\bar{Z}||_{L^\infty}+|\bar{y}-y|||\bar{Z}||_{L^\infty}+|1-x-y|||Z-\bar{Z}||_{L^\infty} \\
                         &\le \epsilon ||\bar{Z}||_{L^\infty} + \epsilon ||\bar{Z}||_{L^\infty}+|1-x-y|\epsilon \\
    &\le 3\epsilon
  \end{align*}
  Finally,
  $$||({e_d})||_{L^\infty}\le ||({e_d})_X||_{L^\infty}+||({e_d})_Y||_{L^\infty}+||({e_d})_Z||_{L^\infty} \le 2\epsilon + 2\epsilon + 3\epsilon = 7 \epsilon.$$
\qed
\end{proof}

With these lemmas, we can prove Proposition \ref{pibar_pi_prop}.

\begin{proof} (of Proposition \ref{pibar_pi_prop})

  Given the assumptions, we will prove that for each $i$,
  $$||(\pi(p_+,p_-,C)K)_i-(\bar{\pi}(\bar{p}_+,\bar{p}_-,\bar{C})\bar{K})_i||_{L^\infty} \le \epsilon.$$
  Let us examine the general case $i>0$, from which the special case $i=0$ (top of the stack) will follow.

  Fix $i>0$ and let
  $$V = p_+K_{i-1}+p_-K_{i+1}+(1-p_+-p_-)K_i,$$
  $$\bar{V} = \bar{p}_+\bar{K}_{i-1}+\bar{p}_-\bar{K}_{i+1}+(1-\bar{p}_+-\bar{p}_-)\bar{K}_i.$$
  Note that
  $$((\pi(p_+,p_-,C)K)_i)_j = h_H\left(V_j-\frac12\right)$$
  and
  $$(\bar{\pi}(\bar{p}_+,\bar{p}_-,\bar{C})\bar{K})_i = \bar{V}.$$
  By Lemma \ref{stack_op_lem_2}, provided $\epsilon<\frac1{14}$, it suffices to show that
  $$||V-\bar{V}||_{L^\infty} \le 7 \epsilon.$$
  This follows directly from Lemma \ref{stack_op_lem_1} with $x=p_+$, $X=K_{i-1}$, $y=p_-$, $Y=K_{i+1}$, $Z=K_i$ and the corresponding choices for the barred quantities.

  The special case $i=0$ (top of the stack) can be proved the same way by replacing $K_{i-1}$ with $C$, $K_{i+1}$ with $K_1$, and $K_i$ with $K_0$.
\qed
\end{proof}

In tensor products that will be used momentarily, if any factor is the zero vector $\vec{0}$, the entire expression becomes zero. Therefore, instead of using the top of the stack $K_0$ directly, we will instead use the stack reading vector $R$, which has one additional component that is approximately $1$ when $K_0$ is approximately $\vec{0}$.
\begin{definition}\label{R_def}\emph{(Stack Reading)}
Given the top vector of the stack $K_0$ with dimension $m_2$, we define the stack reading vector $R$ to be a vector of dimension $m_2+1$ with components
$$
R_i=\left\{
\begin{array}{ll}
    (K_0)_i & 0\le i<m_2 \\
    1-||K_0||_{L^\infty} & i=m_2.
\end{array}\right.
$$
\end{definition}

These two vectors are equivalent in the following sense.
\begin{lemma}\label{K_R_lem}
Given a top-of-stack vector $K_0$ and an idealized top-of-stack vector $\bar{K}_0$, let $R$ and $\bar{R}$ be the corresponding stack reading vector and idealized stack reading vector. Then
$$||K_0-\bar{K}_0||_{L^\infty} = ||R-\bar{R}||_{L^\infty}.$$
\end{lemma}
\begin{proof}
Since it is clear that 
$$||K_0-\bar{K}_0||_{L^\infty} \le ||R-\bar{R}||_{L^\infty},$$
it will suffice to show
$$||R-\bar{R}||_{L^\infty} \le ||K_0-\bar{K}_0||_{L^\infty},$$
which only requires us to confirm that
$$|R_{m_2}-\bar{R}_{m_2}| \le ||K_0-\bar{K}_0||_{L^\infty}.$$
We do this by examining two cases, depending on whether $\bar{R}_{m_2}=1$ or $\bar{R}_{m_2}=0$.

If $\bar{R}_{m_2}=1$, then $\bar{K}_0=\vec{0}$, so
$$|R_{m_2}-\bar{R}_{m_2}|=|(1-||K_0||_{L^\infty})-1|=||K_0||_{L^\infty}=||K_0-\vec{0}||_{L^\infty}=||K_0-\bar{K}_0||_{L^\infty}.$$

If instead $\bar{R}_{m_2}=0$, then there is some index $\hat{i}<m_2$ for which $\bar{R}_{\hat{i}}=(\bar{K}_0)_{\hat{i}}=1$. We have
$$|R_{m_2}-\bar{R}_{m_2}|=|R_{m_2}|=1-||K_0||_{L^\infty} \le 1-(K_0)_{\hat{i}} = (\bar{K}_0)_{\hat{i}}-(K_0)_{\hat{i}} \le ||K_0-\bar{K}_0||_{L^\infty}.$$
\qed
\end{proof}

\subsection{The nnPDA}\label{nnpda_def_sec}

We now turn to the goal of approximating a PDA using a recurrent neural network architecture. Recall Definition \ref{pda_vec_def}, which defined a vector representation of a PDA. We define the \emph{neural network pushdown automaton} (nnPDA) in an analogous manner.


\begin{definition}\label{nnpda_def}\emph{(nnPDA)}
  At any time $t$, the full state of a neural network pushdown automaton is represented by the pair
  $$(Q^t, K^t),$$
  where $Q^t\in[0,1]^n$ is a vector encoding the state and $K^t$ is a differentiable stack according to Definition \ref{continuous_stack_def}.

  The next state pair $(Q^{t+1},K^{t+1})$ is given by the dynamic relations
  \begin{align}
    Q^{t+1}{}_i &= h_H\left((W_Q^{t+1}\cdot Q^t)_i-\frac12\right) \label{nnpda_dynamics_start_eqn} \\
    K^{t+1} &= \pi^{t+1}K^t, \label{nnpda_dynamics_end_eqn}
  \end{align}
  where the transition matrix $W_Q^{t+1}$ and operator $\pi^{t+1}$ are defined as follows.

  Letting $I^{t+1}$ be a one-hot encoding of the $(t+1)$th input character, the transition matrix $W_Q^{t+1}$ is given in component form by
  $$(W_Q^{t+1})_i{}^j = \sum_{k,l}(W_Q)_i{}^{jkl}R^t{}_l I^{t+1}{}_k,$$
  where the components $(W_Q)_i{}^{jkl}$ are chosen to represent the transition rules as in Definition \ref{pda_vec_def},\footnote{Actually, it suffices to use weights that are close to the ideal weights in Definition \ref{pda_vec_def}. We keep them the same in this paper to avoid unnecessary complication.} and the stack reading $R^t$ is determined from $K_0^t$ as described by Definition \ref{R_def}.

  The operator $\pi^{t+1}$ is given by
  $$\pi^{t+1}= \pi(p_+^{t+1}, p_-^{t+1}, C^{t+1}),$$
  where the parameters $p_\pm^{t+1}$ and vector $C^{t+1}$ are
  $$p_{\pm}^{t+1} = \sum_{j,k,l} (W_{p_\pm})^{jkl}R^t{}_lI^{t+1}{}_kQ^t{}_j$$
  $$C^{t+1}{}_i = \sum_{j,k,l}(W_C)_i{}^{jkl}{}R^t{}_lI^{t+1}{}_kQ^t{}_j$$
  where the components $(W_{p_+})^{jkl}$, $(W_{p_-})^{jkl}$, and $(W_C)_i{}^{jkl}$ are chosen to represent the stack action rules as in Definition \ref{pda_vec_def}.
\end{definition}

This definition is illustrated in Figure \ref{nnpda_architecture_fig}.

\begin{figure}
  \caption{\textbf{Neural network pushdown automaton architecture} The nnPDA takes the $(t+1)$th input character represented by $I^{t+1}$ as well as the state represented by $Q^t$ and a reading $R^t$ of the top of the stack determined from $K^t_0$. By applying the weights $W_Q, W_C, W_{p_+}, W_{p_-}$, it computes and outputs the new state represented by $Q^{t+1}$ as well as the parameters for the stack action, represented by $C^{t+1}$, $p_+^{t+1}$, and $p_-^{t+1}$. The stack is then updated by applying the operator $\pi(C^{t+1}, p_+^{t+1}, p_-^{t+1})$. This completes one cycle.}
  \includegraphics[width=0.8\textwidth]{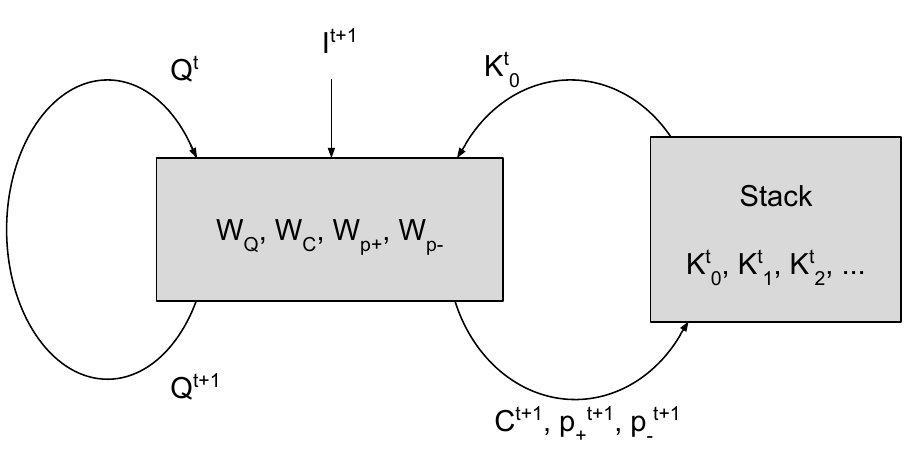}\centering
  \label{nnpda_architecture_fig}
\end{figure}

\subsection{Proof of stability of the nnPDA}\label{nnpda_stability_sec}

The following theorem states that the nnPDA remains close to the idealized PDA, no matter which input string is consumed.

\begin{theorem}\label{thm:nnpda_stabality}
  Let $I^t$ be an arbitrary time-indexed sequence encoding a character string, let $(Q^t,K^t)$ be a state-and-stack sequence governed by equations (\ref{nnpda_dynamics_start_eqn}-\ref{nnpda_dynamics_end_eqn}) in Definition \ref{nnpda_def} for some scalar $H$, and let $(\bar{Q}^t,\bar{K}^t)$ be an idealized state-and-stack sequence governed by equations (\ref{pda_dynamics_start_eqn}-\ref{pda_dynamics_end_eqn}) in Definition \ref{pda_vec_def}.
  
  For any $\epsilon>0$ sufficiently small, if $H$ is sufficiently large (depending on $\epsilon$) and
  $$(Q^{t=0},K^{t=0})=(\bar{Q}^{t=0},\bar{K}^{t=0}),$$
  then for all $t\ge 0$,
  $$||Q^t-\bar{Q}^t||_{L^\infty} \le \epsilon$$
  and
  $$\sup_i||K_i^t-\bar{K}_i^t||_{L^\infty} \le \epsilon.$$
\end{theorem}

This theorem follows by induction on $t$, where the base case ($t=0$) is trivially satisfied by the assumptions, and the inductive step (from $t$ to $t+1$) is addressed by the following proposition.

\begin{proposition}\label{nnPDA_stability_prop}
  Let $(Q^t,K^t)$ be a state-and-stack sequence governed by equations (\ref{nnpda_dynamics_start_eqn}-\ref{nnpda_dynamics_end_eqn}) in Definition \ref{nnpda_def} for some scalar $H$, and let $(\bar{Q}^t,\bar{K}^t)$ be an idealized state-and-stack sequence governed by equations (\ref{pda_dynamics_start_eqn}-\ref{pda_dynamics_end_eqn}) in Definition \ref{pda_vec_def}, and let $e_d$ be the difference between differentiable and discrete operators.

  For any $\epsilon>0$ sufficiently small, if $H$ is sufficiently large and
  $$||Q^t-\bar{Q}^t||_{L^\infty} \le \epsilon$$
  $$\sup_i||K_i^t-\bar{K}_i^t||_{L^\infty} \le \epsilon,$$
  then
  $$||Q^{t+1}-\bar{Q}^{t+1}||_{L^\infty} \le \epsilon$$
  $$\sup_i||K_i^{t+1}-\bar{K}_i^{t+1}||_{L^\infty} \le \epsilon.$$
\end{proposition}
\begin{proof}
  First, we will show that
  $$||Q^{t+1}-\bar{Q}^{t+1}||_{L^\infty} \le \epsilon.$$
  Note that
  $$Q^{t+1}{}_i = h_H\left(V_i-\frac12\right),$$
  where
  $$V_i = \sum_{j,k,l}(W_Q)_i{}^{jkl}R^t{}_lI^{t+1}{}_kQ^t{}_j,$$
  And
  $$\bar{Q}^{t+1}{}_i = \bar{V}_i,$$
  where
  $$\bar{V}_i = \sum_{j,k,l}(W_Q)_i{}^{jkl}\bar{R}^t{}_lI^{t+1}{}_k\bar{Q}^t{}_j.$$
  Observe that
  $$V=\bar{V}+{e_d},$$
  where
  \begin{align*}
    {e_d} &= ({e_d})_1+({e_d})_2 \\
    ({e_d})_1 &= \sum_{j,k,l}(W_Q)_i{}^{jkl}\bar{R}^t{}_lI^{t+1}{}_k \left[Q^t{}_j-\bar{Q}^t{}_j\right]\\
    ({e_d})_2 &= \sum_{j,k,l}(W_Q)_i{}^{jkl}\left[R^t{}_l-\bar{R}^t{}_l\right]I^{t+1}{}_k Q^t{}_j.
  \end{align*}
  The factors in square brackets are each at most size $\epsilon$,\footnote{The difference $R-\bar{R}$ relates to the difference $K_0-\bar{K}_0$ by Lemma \ref{K_R_lem}.} while the remaining factors are at most size $1$. It follows that for some constant $C>0$,
  $$||V-\bar{V}||_{L^\infty}\le C\epsilon.$$
  By Lemma \ref{stack_op_lem_2}, we may conclude that if $\epsilon$ is sufficiently small and $H$ is sufficiently large (depending on $\epsilon$), then
  $$||Q^{t+1}-\bar{Q}^{t+1}||_{L^\infty} \le \epsilon.$$

  By similar arguments, we may also prove that, with possibly additional constraints on $H$,
  $$|p_+^{t+1}-\bar{p}_+^{t+1}|\le \epsilon$$
  $$|p_-^{t+1}-\bar{p}_-^{t+1}|\le \epsilon$$
  $$||C^{t+1}-\bar{C}^{t+1}||_{L^\infty} \le \epsilon.$$
  Given these estimates, by Proposition \ref{pibar_pi_prop}, it follows that
  $$\sup_i||(\pi(p_+^{t+1},p_-^{t+1},C^{t+1})K^t)_i-(\bar{\pi}(\bar{p}_+^{t+1},\bar{p}_-^{t+1},\bar{C}^{t+1})\bar{K}^t)_i||_{L^\infty}\le \epsilon,$$
  which means
  $$\sup_i||K_i^{t+1}-\bar{K}_i^{t+1}||_{L^\infty} \le \epsilon.$$
\qed
\end{proof}

We conclude this section by rewriting Theorem \ref{thm:nnpda_stabality} in a form that more closely resembles the statement in \cite{siegelmann95}. But in particular, we note that the work in this paper does not require unbounded precision and time.

\begin{theorem}
  For any given Pushdown Automaton (PDA) $M$ with n states and m stack symbols, there exists a differentiable nnPDA with n+1 bounded precision neurons that can simulate $M$ in real-time. 
\end{theorem}

\section{Neural Network with Tape Memory}

In this section, we generalize the work done in the previous section from Pushdown Automata to Turing Machines. Most of the heavy lifting has already been done at this point--our approach will be to represent a tape, the memory associated with a Turing Machine, as a pair of stacks.

\subsection{The Tape}\label{tm_tape_sec}
A Turing Machine is a DFA augmented with a tape. The following definition closely follows \cite{hopcroft2001introduction}.

\begin{definition}\emph{(Turing Machine, classical version)}
A Turing machine is a 7-tuple defined as:  $M=\langle Q,\Gamma,b,\Sigma,\delta,q_s,F \rangle$, where $Q=\{q_1(=q_s),q_2,\ldots,q_n\}$ is the set of finite states, $\Gamma =\{s_0,s_1,\ldots,s_m\}$ is the input alphabet set including $b=s_0$, $\Sigma$ is a finite set of tape symbols, $\delta: \Gamma\times Q \to \Gamma\times Q\times op(\Gamma)$ is the transition function or rule for the grammar, $q_s \in Q$ is the initial starting state and $F \subset Q$ is the set of final states. 
\end{definition}
The instantaneous configuration of a Turing Machine is typically defined as a tuple of state, tape, and the location of the read/write head. Typically, the tape is described as an array of characters without end on either side and the head as a signed index into this array. To take advantage of the work done in the previous section, we prefer an alternate representation in which the tape is split into two stacks, the right stack and the left stack. The right stack, denoted as $k$, consists of all symbols on the tape under or to the right of the head. The left stack, denoted as $l$, consists of all symbols left of the head. The top of each stack is the closest to the tape head and the blank symbols of the tape on both sides are omitted as they would be below the bottom of either stack. Although it is clear the stack pair encodes the memory of a tape, in order to ensure that the it properly represents a tape, we must pay careful attention to the operations we apply to each stack simultaneously.


\begin{definition}\emph{(Tape operators, classical version)}
Define the set of tape operators
$$op(\Gamma)=\{\omega_-(\charb_1),...,\omega_-(\charb_{m_2}), \omega_0, \omega_+\}.$$
These operators have the following effect on the tape: $\omega_-(\charc)$ writes $\charc$ and then moves the head to the right, $\omega_0$ has no effect, and $\omega_+$ moves the head to the left.
\end{definition}
In particular, note that while the stack operators parametrize $\pi_+$ and have a single $\pi_-$, instead for the tape operators $\omega_+$ and $\omega_-$, this is reversed. 
There is an asymmetry here--we must pick one stack to have its top lie under the head. By convention, we choose to have the head read from $k$ to be consistent with the PDA. We would also like $\omega_-$ to act on $k$ the same way $\pi_-$ acts on $k$ and $\omega_+$ the same way as $\pi_+$. But for this to happen, we must parametrize $\omega_-$ instead of $\omega_+$. Otherwise, there would be no way to read from $l$. Instead, we indirectly read from the left stack when applying $\pi_+$ to $k$.

The operators $\omega_-(\charc)$, $\omega_0$, and $\omega_+$ act on a double-stack tape the following way.
\begin{align}
\omega_-(\charc)(k, l) &= (\pi_-k, \pi_+(\charc) l)\label{stack_op_start_eqn} \\
\omega_0(k, l) &= (k, l) \\
\omega_+(k, l) &= (\pi_+(l_0)k, \pi_- l).\label{stack_op_end_eqn}
\end{align}

With the stack pair representation of a tape, it is now straightforward to generalize to a differentiable tape using the work developed in the previous section.
\begin{definition}\label{continuous_tape_def}\emph{(Differentiable tape)}
A differentiable tape is a pair of differentiable stacks $(K, L)$, each as defined in Definition \ref{continuous_stack_def}. The stack $K$ is called the \emph{right stack} and the stack $L$ is called the \emph{left stack}. The \emph{head} of the tape reads $K_0$, the top of the stack $K$.
\end{definition}

The following is the differential analogue of the operator $\omega$.
\begin{definition}\emph{(Differentiable tape operator)}
For scalars $p_+$, $p_-$ and vector $C\in[0,1]^n$, we define the tape operator $\omega(p_+,p_-,C)$ as follows.
$$\omega(p_+, p_-, C)(K, L) = (\pi(p_+, p_-, L_0)K, \pi(p_-, p_+, C)L).$$
\end{definition}
Note in particular that the arguments $p_\pm$ for the operator acting on $L$ are reversed, and the operator acting on $K$ has the top of the left stack as its third argument. This way, $\omega(1, 0, C)$, $\omega(0, 0, C)$, and $\omega(0, 1, C)$ approximate $\omega_+$, $\omega_0$, and $\omega_-(\charc)$ respectively.

We now fully define the nnTM architecture. This is directly analogous to Definition \ref{nnpda_def}.
\begin{definition}\label{nntm_def}\emph{(nnTM)}
  At any time $t$, the full state of a neural network Turing Machine is represented by the triplet
  $$(Q^t, K^t, L^t),$$
  where $Q^t\in[0,1]^n$ is a vector encoding the state and $K^t$ and $L^t$ are stacks representing respectively the right and left part of a differentiable tape according to Definition \ref{continuous_tape_def}.

  The next state pair $(Q^{t+1}, K^{t+1}, L^{t+1})$ is given by the dynamic relations
  \begin{align}
    Q^{t+1}{}_i &= h_H\left((W_Q^{t+1}\cdot Q^t)_i-\frac12\right) \label{nntm_dynamics_start_eqn} \\
    (K^{t+1}, L^{t+1}) &= \omega^{t+1}(K^t, L^t), \label{nntm_dynamics_end_eqn}
  \end{align}
  where the transition matrix $W_Q^{t+1}$ and operator $\omega^{t+1}$ are defined as follows.

  Letting $I^{t+1}$ be a one-hot encoding of the $(t+1)$th input character, the transition matrix $W_Q^{t+1}$ is given in component form by
  $$(W_Q^{t+1})_i{}^j = \sum_{k,l}(W_Q)_i{}^{jkl}R^t{}_l I^{t+1}{}_k,$$
  where the components $(W_Q)_i{}^{jkl}$ are chosen to represent the transition rules as in Definition \ref{pda_vec_def}, and the stack reading $R^t$ is determined from $K_0^t$ as described by Definition \ref{R_def}.

  The operator $\omega^{t+1}$ is given by
  $$\omega^{t+1}= \omega(p_+^{t+1}, p_-^{t+1}, C^{t+1}),$$
  where the parameters $p_\pm^{t+1}$ and vector $C^{t+1}$ are
  $$p_{\pm}^{t+1} = \sum_{j,k,l} (W_{p_\pm})^{jkl}R^t{}_lI^{t+1}{}_kQ^t{}_j$$
  $$C^{t+1}{}_i = \sum_{j,k,l}(W_C)_i{}^{jkl}{}R^t{}_lI^{t+1}{}_kQ^t{}_j$$
  where the components $(W_{p_+})^{jkl}$, $(W_{p_-})^{jkl}$, and $(W_C)_i{}^{jkl}$ are chosen to represent the tape action rules as in Definition \ref{pda_vec_def}.
\end{definition}

The following theorem now follows from Theorem \ref{thm:nnpda_stabality} and our formulation of the Turing Machine as a machine with two stacks.
\begin{theorem}
  Let $I^t$ be an arbitrary time-indexed sequence encoding a character string, let $(Q^t, K^t, L^t)$ be a state-and-tape sequence governed by equations (\ref{nntm_dynamics_start_eqn}-\ref{nntm_dynamics_end_eqn}) in Definition \ref{nntm_def} for some scalar $H$, and let $(\bar{Q}^t,\bar{K}^t,\bar{L}^t)$ be an idealized state-and-tape sequence representing a vectorized Turing Machine.

  For any $\epsilon>0$ sufficiently small, if $H$ is sufficiently large (depending on $\epsilon$) and
  $$(Q^{t=0}, K^{t=0}, L^{t=0})=(\bar{Q}^{t=0},\bar{K}^{t=0},\bar{L}^{t=0}),$$
  then for all $t\ge 0$,
  $$||Q^t-\bar{Q}^t||_{L^\infty}\le \epsilon,$$
  $$\sup_i||K_i^t-\bar{K}_i^t||_{L^\infty}\le \epsilon,$$
  $$\sup_i||L_i^t-\bar{L}_i^t||_{L^\infty}\le \epsilon.$$
\end{theorem}

As a corollary, we may conclude the following.
\begin{theorem}\emph{(Stability of nnTM)}
For any given Turing Machine $M$ with $n$ states and $m$ tape symbols, there exists a differentiable nnTM with bounded precision neurons that can accurately represent $M$ for arbitrarily long input strings.
\end{theorem}

\cite{neary2007four} constructed four Universal Turing Machines (UTMs)
$$UTM_{n,m}\in\{UTM_{5,5}, UTM_{6,4}, UTM_{9,3}, UTM_{18,2}\}$$
that can simulate any TM.
\begin{theorem}(\cite{neary2007four})
 Given a deterministic single tape Turing machine $M$ that runs in time $t$ then any of the above UTMs can simulate the computation of $M$ using space $O(n)$ and time $O(t^2)$.
\end{theorem}
In light of this work and the previous theorem, the following stronger theorem holds.
\begin{theorem}
A nnTM with $n=6$ states and $m=4$ tape symbols can stably simulate any TM in $O(t^2)$ time.
\end{theorem}
The above theorem states that nnTM with only $6$ state and $1$ output neuron can simulate any TM in $O(t^2)$ time and solution constructed by differentiable nnTM will always stay stable.

It is worth noting (\cite{siegelmann94}) construction would require minimum 40 unbounded precision state neurons to show equivalence with UTM, such that RNN can simulate any TM in $O(t^8)$ time. Similarly with respect to bounded neuron augmented with memory (\cite{chung2021turing}) construction would require minimum 54 bounded precision state neurons augmented with 2 discrete stack like memory (growing memory) to show their construction is similar to UTM such that model can can simulate any TM in $O(t^6)$ time. Furthermore other form of neural networks such as transformers and Neural GPU's (\cite{attn_turing}) cannot simulate UTM's. All prior constructions do not guarantee stability. 

\section{Discussion and Complexity}

The operations performed by our tensor nnTM directly map to TM state transitions, making the TM encoding straightforward. For a state machine with $n$ states, an input alphabet of size $m_1$, a tape alphabet of size $m_2$, and a tape with maximum capacity $s$, we show how to construct a sparse recurrent network with $n+1$ state neurons, $O(m_1m_2n(m_1+m_2+n))$ weights and a tape memory footprint of size $O(sm_2)$, such that the TM and constructed network accept recursively enumerable languages.

\section{Conclusion}

We defined a vectorized stack memory and a parameterized stack operator that behaves similarly to stack push/pop operations for particular choices of parameters. We used these to construct a neural network pushdown automaton (nnPDA). We showed that our construction is stable: for suitable choices of weights, the nnPDA will closely resemble a corresponding PDA for arbitrarily long strings. (The sense in which the nnPDA and PDA are close is formalized by representing both as vectors and taking the vector difference.)

We then represent a tape as a pair of stacks and adapt our parametrized memory operator, thereby extending our result to a neural network Turing Machine (nnTM) as a stable approximation of a Turing Machine.

We observe that, in light of work by \cite{neary2007four}, this means a nnTM architecture with $n=6$ states and $m=4$ tape symbols can simulate any TM.

\bibliographystyle{IEEE}
\bibliography{jacm_bib}
\end{document}